\documentclass[letterpaper, 10 pt, conference]{ieeeconf}

\IEEEoverridecommandlockouts                          
\usepackage[T1]{fontenc}
\usepackage[utf8]{inputenc}
\usepackage[T1]{fontenc}
\usepackage[pagebackref=false,breaklinks=true,colorlinks,bookmarks=false,citecolor=citecolor,linkcolor=linkcolor]{hyperref}
\usepackage{url}
\usepackage{booktabs}
\usepackage{amsfonts}
\usepackage{nicefrac}
\usepackage{microtype}
\usepackage{xcolor}
\usepackage{amsmath, amssymb}

\usepackage{amsthm,mathtools}
\usepackage{pgfplots}
\usepackage{tikz}
\usepackage{bbm}
\usepackage{tikz-3dplot}
\usepackage[normalem]{ulem}
\usepackage{algorithm2e}
\usepackage{dsfont}

\usetikzlibrary{arrows.meta}
\usetikzlibrary{patterns}
\usepackage{multirow, multicol}
\let\labelindent\relax
\usepackage{enumitem}
\usepackage{tabu}
\usepackage{comment}

\usepackage{cite}

\usepackage{algorithm}
\usepackage{subcaption}
\usepackage{graphicx}

\newcommand{\E}{\mathbb{E}}
\newcommand{\vol}{\textrm{Vol}}
\usepackage{graphicx}
\newtheorem{assumption}{Assumption}

\newtheorem{definition}{Definition}
\newtheorem{theorem}{Theorem}
\newtheorem{proposition}{Proposition}
\newtheorem{lemma}{Lemma}
\newtheorem{example}{Example}
\newtheorem{corollary}{Corollary}
\newtheorem{remark}{Remark}
\newtheorem{claim}{Claim}
\definecolor{citecolor}{HTML}{0071BC}
\definecolor{linkcolor}{HTML}{ED1C24}

\newcommand{\R}{\mathbb R}

\DeclareMathOperator*{\argmax}{arg\,max}

\title{\LARGE \bf
Set Membership Estimation for Stochastic Linear Contextual Bandits
}

\author{Haonan Xu and Yingying Li
\thanks{Haonan Xu and Yingying Li are with the Coordinated Science Laboratory,
        University of Illinois at Urbana-Champaign, Urbana, IL 61801, USA.
        Email:
        {\tt\small \{haonan9,yl101\}@illinois.edu}}}

\begin{document}

\maketitle
\thispagestyle{empty}
\pagestyle{empty}

\begin{abstract}
Contextual bandit learning has attracted increasing attention in engineering applications due to its ability to incorporate side/context information in time-varying  environments. A widely used approach for contextual bandits is optimism in the face of uncertainty (OFU), in which uncertainty sets for unknown parameters are constructed and decisions are made optimistically over these sets. In the bandit literature,  uncertainty sets are typically derived using statistical learning tools, such as confidence intervals. In contrast, set membership estimation (SME), a well-established uncertainty quantification approach in control and system identification, has received  little attention in  the bandit literature. This paper investigates the connection between SME and contextual bandits by focusing on linear contextual bandits as a  starting point. We propose an SME-based OFU algorithm and establish a sublinear regret bound under proper assumptions. Simulations  further demonstrate the promising performance of SME-OFU compared with existing methods.

\textbf{This supplementary manuscript provides more  proof details in Section \ref{sec: proof sketch} and the Appendix.\\}
\end{abstract}


\section{Introduction}\label{sec:intro}
Linear contextual bandits (LCB) provide a framework for leveraging side/contextual information in online decision making and have been widely adopted in many applications, e.g.  smart grids \cite{jin2025privacy,bregere2019target}, mobile health \cite{lei2017actor, greenewald2017action}, robotics \cite{madan2025prioritouch}, wireless communication \cite{thornton2021constrained}, etc. In LCB, an agent chooses from $K$ arms, and each arm $a$ is associated with a context information vector $X_{t,a}\in \R^d$ revealed by the environment. At each stage $1\leq t \leq T$, after choosing arm $a_t$, the agent observes a noisy linear reward:
\begin{equation*}
r_{t,a_t}=\theta_*^\top X_{t,a_t}+\eta_t 
\end{equation*} where $\theta_*$ is an unknown parameter vector for all arms and $\eta_t$ denotes independent random noise. 

One popular algorithmic framework for LCB is  the optimism in the face of uncertainty (OFU) principle \cite{abbasi2011improved,guo2019adalinucb,li2021tight,chu2011contextual,kazerouni2017conservative,li2010contextual,dani2008stochastic,flynn2023improved}. OFU-based algorithms  construct  uncertainty sets of the unknown parameters $\theta_*$ and select arms optimistically by maximizing the predicted reward over the admissible parameters in the uncertainty set. The  construction of uncertainty sets is crucial for the exploration-exploitation tradeoff in bandit learning: over-conservative uncertainty sets can induce unnecessary exploration and hinder reward maximization, while  overly small uncertainty sets may fail to contain the true model, which violates the optimism principle and  causes suboptimal exploitation \cite{abbasi2011improved,flynn2023improved}. Existing OFU-based algorithms typically construct  uncertainty sets with  confidence bounds for  least-squares estimators (LSE) \cite{abbasi2011improved}, or modifications of  LSE confidence  bounds \cite{chu2011contextual,guo2019adalinucb,li2021tight,flynn2023improved}.

Most OFU-based LCB literature considers sub-Gaussian reward noise $\eta_t$  but assumes bounded expected rewards $\theta_*^\top X_{t,a}$ for all $t\geq 1$ and $1\leq a\leq K$ \cite{abbasi2011improved,li2021tight,guo2019adalinucb}. However, in many applications, the observed reward $r_{t,a}$  is naturally bounded, implying bounded noise $\eta_t$. For example, in demand response (DR), the observed reward $r_{t,a}$ can be modeled as the amount of reduced electricity consumption in a DR event, which is always bounded \cite{jin2025privacy,bregere2019target}. Similarly, in mobile health, the observed reward may represent the patient's step count over a few hours,  which is also bounded \cite{greenewald2017action,lei2017actor}. Therefore, it is reasonable to  consider bounded random noise in  LCB, which is a stronger property than sub-Gaussian noise.

A promising uncertainty quantification method for bounded-noise settings is set membership estimation (SME).  SME has  been extensively studied  in the control literature \cite{milanese1991optimal,bai1998convergence,fogel1982value,akccay2004size, bai1999convergence,lorenzen2019robust, zhang2021trajectory}, especially in robust adaptive model predictive control, for quantifying and reducing system uncertainties to enable  robustness guarantees \cite{lorenzen2019robust,lu2021robust,zhang2021trajectory}. Traditionally, SME is applied to system identification with deterministic bounded noise. More recently, it has been revealed that SME also provides promising numerical and theoretical performance for system identification under stochastic bounded noise \cite{li2024learning,zeng2025system}. This has motivated a growing interest in applying SME to modern learning-based control  and online control \cite{xu2025sample,yu2023online,musavi2024identification,tang2024uncertainty,musavi2026finite}.  Nevertheless, existing applications of SME remain largely focused on control-oriented problems. Few studies have applied SME to online bandit learning or examined its effects on  the exploration and exploitation tradeoff.

\noindent\textbf{Our Contributions.} This paper connects SME with LCB by proposing an SME-OFU algorithm and establishing a sublinear regret bound. To the best of our knowledge, this is the first work to apply SME to bandit learning with provable regret guarantees. Our regret analysis inherits most of the standard assumptions used in  SME convergence rate analysis, such as knowledge of a tight  bound on noise $\eta_t$ \cite{akccay2004size,li2024learning,lu2021robust,zeng2025system}, while relaxing the persistent excitation  (PE) condition. This relaxation is particularly important in LCB because the contexts are provided by the environment rather than  arbitrarily chosen by the agent, in contrast to system identification where inputs can often be designed to ensure PE. To relax the PE assumption, we develop a novel analysis framework based on convex geometry by leveraging minimum-volume enclosing ellipsoids and  weighted widths to characterize the imbalanced evolution of SME uncertainty sets. Simulation results further show promising empirical performance compared with existing algorithms such as LinUCB \cite{chu2011contextual} and OFUL \cite{abbasi2011improved}. Overall, our proposed algorithm and analysis lay a foundation for future work on SME-based bandit learning, including extensions to other SME and bandit variants and  relaxation of the assumption on known tight noise bounds. 

\vspace{2pt}

\noindent\textbf{Related Work.} LCB mainly has two formulations: stochastic LCB and adversarial LCB. This paper focuses on the stochastic setting.  In addition to OFU-based algorithms \cite{abbasi2011improved,guo2019adalinucb,li2021tight,chu2011contextual,kazerouni2017conservative,li2010contextual,dani2008stochastic}, Bayesian approaches are also adopted in stochastic LCB, such as  LinTS   and its variants \cite{agrawal2013thompson,dimakopoulou2019balanced,xu2023noise}. Further, some stochastic LCB literature assumes i.i.d. contexts \cite{dimakopoulou2019balanced, wu2020stochastic}, while other literature, including this paper, allows   contexts to be generated by an oblivious adversary \cite{abbasi2011improved, chu2011contextual,agrawal2013thompson, xu2023noise}. 
For adversarial LCB, bounded realized rewards $r_{t,a}$ are  considered but the unknown parameters $\theta_t$ can be time-varying \cite{liu2023bypassing,neu2020efficient}. In addition to LCB, there are other settings, such as nonlinear contextual bandits \cite{guo2026signature} and neural contextual  bandits \cite{zhou2020neural}. An interesting future direction is to apply SME to these contextual learning settings.
\section{Problem Formulation}\label{sec:formulation}
This section first introduces the linear contextual bandit formulation, followed by two engineering applications. Then, we discuss the assumptions adopted in this paper and review two algorithms, OFU and SME, in preparation for our algorithm design in Section \ref{sec: alg}.

\subsection{Linear Contextual Bandit Formulation}
This paper studies the following linear contextual bandit (LCB) problem. Consider a horizon of   $T$ stages and arm choices $\{1,\dots, K\}$.  At each stage $1\leq t\leq T$,   the agent receives a context vector $X_{t,a}\in \mathbb R^d$ for each arm $1\leq a\leq K$. The contexts can be time-varying and even generated by an oblivious adversary. Based on the context information, the agent selects an arm $a_t$ and observes a noisy reward of the selected arm
\begin{equation}
    \label{eq: cb reward}
    r_{t,a_t} = \theta_*^\top X_{t,a_t}+\eta_t,
\end{equation}
where $\theta_*\in\mathbb{R}^d$ is an unknown  parameter vector and   $\eta_t \in \mathbb R$ is zero-mean random noise. The expected reward $$\E_{\eta_t}[r_{t,a_t}]= \theta_*^\top X_{t,a_t}$$ is linear in the context vector $X_{t,a_t}$, hence the name linear contextual bandits.\footnote{Here, we follow the convention of most LCB literature and use a shared $\theta_*$ for all arms \cite{guo2019adalinucb,li2021tight,abbasi2011improved,chu2011contextual}. When the model is arm-dependent, i.e. $\E_{\eta_t}[r_{t,a}]= \theta_{*,a}^\top X_{t,a}$, it can be reformulated  by concatenation, i.e., let $\theta_*=(\theta_{*,1}^\top, \dots, \theta_{*,K}^\top)^\top$ and let the context be $(0, \dots,0, X_{t,a}^\top, 0, \dots, 0)^\top$.}

Following  the LCB literature \cite{guo2019adalinucb,li2021tight,abbasi2011improved,chu2011contextual},
we consider a bounded parameter vector $\theta_*$ and  bounded context vectors $X_{t,a}$, i.e.,  there are finite upper bounds $\theta_{\max},x_{\max}\in (0, +\infty)$ such that
$\|\theta_*\|_\infty \leq \theta_{\max} $  and $\|X_{t,a}\|_2\leq x_{\max}$ for all $1\leq t\leq T$ and  $1\leq a \leq K$. 

The goal  is to 
maximize the total expected reward across $T$ stages without knowing $\theta_*$ beforehand. The  performance is measured by \textit{regret}, which is defined as
\begin{equation}\label{eq:pseudo-regret}
    \textup{Reg}(T) := \sum_{t=1}^T \left(\theta_*^\top X_{t,a_t^*} - \theta_*^\top X_{t,a_t}\right),
\end{equation}
where $a_t^* := \arg\max_{a\in[K]} \theta_*^\top X_{t,a}$ is a time-varying optimal arm at stage $t$ when $\theta_*$ is known,  and $a_t$ is selected by an LCB algorithm without knowing $\theta_*$. Therefore,  regret  \eqref{eq:pseudo-regret} captures the performance loss due to unknown $\theta_*$.  Since \eqref{eq:pseudo-regret} focuses on expected rewards instead of realized rewards, it is also called  expected regret or pseudo regret \cite{abbasi2011improved}. 

\subsection{Application Examples}
Linear contextual bandits are adopted in many applications \cite{jin2025privacy,bregere2019target,lei2017actor,greenewald2017action,madan2025prioritouch,thornton2021constrained}.  Some  examples are provided below. Although the linear reward model is often an approximation, it can still yield satisfactory performance in practice. More general models, such as nonlinear or neural contextual bandits \cite{guo2026signature,zhou2020neural}, are beyond the scope of this work and left for future study.

\begin{example}[Mobile Health]\label{example: mobile health}
LCB have been widely used in mobile health applications \cite{greenewald2017action,lei2017actor}. For example, \cite{greenewald2017action} uses  LCB  to encourage walking by deciding which suggestion to send on a mobile application. The arms are ${0,1,\ldots,K}$, where arm 0 sends no message and arm $i$ sends message $i$. The  context information $\bar X_t$ includes  location, time, weather,  physical activity, and internet activity. The reward is $\log(0.5+s_t)$, where $s_t$ is the user's step count during the following 30 minutes. Although human behavior is nonlinear, \cite{greenewald2017action} models the reward as
$
r_t(a)=g(\bar X_t)+\theta_*^\top X_{t,a}+\eta_t,
$ where $X_{t,a}=f(\bar X_t)\mathbf 1_{(a>0)}$ and
 $f(\cdot)$ is known. Since the nonlinear term $g(\bar X_t)$ is arm-independent, it does not affect action selection, so the problem can be solved by LCB \cite{greenewald2017action}.
\end{example}

\begin{example}[Smart Grids]\label{example: smart grid} 
LCB have also been applied to demand response problems  in smart grids \cite{jin2025privacy,bregere2019target}. For example, \cite{jin2025privacy} treats each power consumer as an arm and the corresponding load reduction as the reward, with the goal of maximizing total load reduction. The context includes factors such as indoor and outdoor temperature, incentive level, user fatigue, and thermal comfort. The expected load reduction is approximated linearly as
$
\E[r_{t,a}] = \theta_*^\top X_{t,a},
$
where $X_{t,a}$ is a known feature representation of the contextual information. Simulations on real-world data show that this linear approximation achieves good numerical performance.
\end{example}

\subsection{Assumptions}
In the following, we discuss two  assumptions on reward noise $\eta_t$  as well as their motivations and implications.

First, we assume bounded noise.
\begin{assumption}[Bounded Noise]\label{ass: noise}
The noise $\eta_t$ is sampled from an unknown distribution and is independent of the historical contexts $\{\!X_{s,a}\!:s\!\leq\!t,\,a\!\leq\!K\!\}$ and noise realizations $\{\eta_s\}_{s=1}^{t-1}$. Further, it satisfies  $\E[\eta_t]=0$ and $|\eta_t| \leq \eta_{\max}<+\infty$, where $ \eta_{\max}$ is known.
\end{assumption}
Assumption \ref{ass: noise} is closely related to the noise conditions commonly used in stochastic LCB. Standard OFU-based analyses
typically assume that the reward noise is conditionally zero-mean and
sub-Gaussian given the information available before observing the
reward~\cite{abbasi2011improved}. In comparison, Assumption~1 imposes
a stronger requirement that the noise has bounded support.
Bounded noise is common in real-world applications. In the mobile health example (Example \ref{example: mobile health}), the reward is the user’s 30-minute step count under a logarithmic transformation, which is naturally bounded; hence $\mathbb{E}_{\eta_t}\left[r_{t,a_t}\right]$ is also bounded, implying bounded noise $\eta_t=r_{t,a}-\E[r_{t,a}]$. Similarly, in Example \ref{example: smart grid}, both the consumer’s load reduction and its expectation are bounded, so the corresponding noise term is bounded as well. Moreover,  upper bounds on the reward noise, especially conservative ones, can often be obtained from domain knowledge.

Next, we  assume a known tight  bound on reward noise. 
\begin{assumption}[Tight Bound]\label{ass: boundary-visiting} There exists a positive value $\epsilon_0>0$ and a  function $q(\cdot)>0$  on $(0, \epsilon_0]$ such that
$$\mathbb{P}(\eta_{\max}-\eta_t\leq\epsilon)\geq q(\epsilon), \quad \ \mathbb{P}(\eta_t+\eta_{\max}\leq\epsilon)\geq q(\epsilon) $$
for any $\epsilon \in(0, \epsilon_0]$ and $t\geq 1$. Further,  $q(\epsilon) \geq c_0\epsilon\ $ for   $c_0>0$.
\end{assumption}

Though Assumption \ref{ass: boundary-visiting} can be quite restrictive, it is only imposed for regret analysis and not necessary for the implementation of our proposed algorithm. Further, Assumption \ref{ass: boundary-visiting} is a standard assumption in the convergence rate analysis of SME \cite{li2024learning,akccay2004size,lu2021robust,zeng2025system}. As a starting point to connect SME with bandit learning, we inherit this assumption from the SME literature for ease of analysis.

Besides, there are attempts to relax Assumption \ref{ass: boundary-visiting} in the SME literature. For example, there are variants of SME that learn a tight noise bound and construct uncertainty sets simultaneously \cite{li2024learning}. It is our ongoing work to incorporate these SME variants to relax Assumption \ref{ass: boundary-visiting} while preserving the regret guarantee.

The linear boundary visiting condition $q(\epsilon) \geq c_0\epsilon\ $  is  satisfied by a broad class of distributions, such as truncated Gaussian and uniform distributions \cite{li2024learning,musavi2026finite}. Such distributions are relevant in many applications. For example, \cite{greenewald2017action} observes that  noise is approximately Gaussian  when using  the logarithm of step count as the reward signal (see Example \ref{example: mobile health}),  and since the step count is naturally bounded, a truncated Gaussian distribution provides a reasonable approximation to such bounded, approximately Gaussian noise.

Further, the symmetric bound on $\eta_t$, i.e., $-\eta_{\max}\leq \eta_t\leq \eta_{\max}$, is only imposed for notational simplicity. Our algorithm design and regret analysis can be easily generalized to the asymmetric case: $\eta_{\min}\leq \eta_t\leq \eta_{\max}$.

Finally, we emphasize that, unlike most literature on SME's convergence rates \cite{akccay2004size,lu2021robust,li2024learning,musavi2024identification,zeng2025system}, we do not impose a persistent excitation (PE) condition on the context vectors. This relaxation is particularly important in the LCB setting because the contexts are generated by the environment and it is beyond the control of the decision maker to achieve PE. 

\subsection{Preliminaries: OFU and SME Algorithms}
In the following, we review two algorithms in preparation for our algorithm design in Section \ref{sec: alg}: OFU for linear contextual bandits and SME for linear regression.

\noindent\textbf{Preliminaries: Optimism in the Face of Uncertainty.}
OFU is a widely used principle for balancing exploration and exploitation in online bandit learning \cite{abbasi2011improved,guo2019adalinucb,li2021tight,chu2011contextual,kazerouni2017conservative,li2010contextual,dani2008stochastic}. In LCB, OFU-based algorithms proceed as follows. At each stage $t$, the agent constructs a confidence set $\Theta_{t-1}$ for the unknown parameter $\theta_*$ using the observations collected before stage $t$. Then, the agent  jointly selects an optimistic model within this set and an arm that maximizes the expected reward under the optimistic model, that is,
\begin{equation}
(\theta_t, a_t) \in \argmax_{\theta \in \Theta_{t-1}, 1 \leq a \leq K} \theta^\top X_{t,a},
\end{equation}
where $\theta_t$ denotes the optimistic model and $a_t$ is the selected arm.

OFU  implicitly balances the exploration-exploitation tradeoff. This is because OFU  selects an arm either because the available data indicate a large expected reward or because its expected reward is highly uncertain and could  potentially be very large. The former promotes exploitation, while the latter encourages exploration. Selecting an arm can generate additional information and thus reduce uncertainty  about the parameters in the directions associated with its context vectors. As more observations are collected, the reward uncertainty decreases, and the agent increasingly selects arms that have achieved high rewards in the past. In this way, OFU balances exploration and exploitation and achieves sublinear regret.

In the LCB literature, the uncertainty sets adopted by OFU-based algorithms are usually constructed with confidence sets for least squares estimation or their variants \cite{chu2011contextual,guo2019adalinucb,li2021tight}.

\vspace{4pt}
\noindent\textbf{Preliminaries:  Set Membership Estimation.}
SME is an uncertainty quantification method for linear regression with bounded noise and enjoys a long history of research in robust control and system identification  \cite{milanese1991optimal,lorenzen2019robust,zhang2021trajectory,bai1998convergence, bai1999convergence,akccay2004size, li2024learning}. 

This paper only considers the vanilla version of SME. In particular, at each stage $t\geq 1$, SME utilizes the noise bound $|\eta_t|\leq \eta_{\max}$ and the historical data $\{X_{s,a_s}, r_{s,a_s}\}_{s=1}^{t}$ to construct the uncertainty set of the parameter $\theta_*$:
\begin{align}
    \label{eq: membership set}
    &\Theta_t = \Theta_{t-1}\cap\left\{\hat\theta\,:\, |r_{t,a_t} - \hat\theta^\top X_{t,a_t}|\leq\eta_{\max}\right\},
\end{align}
where $\Theta_0 = \{\hat \theta\,:\,  \|\hat \theta\|_\infty \leq \theta_{\max}\}$ is the initial uncertainty set. 
It is easy to see that $\theta_*\in \Theta_t$ for all $t\geq 0$ as long as $\eta_{\max}$ is a valid upper bound on the noise term because $|r_{t,a_t} - \theta_*^\top X_{t,a_t}|=|\eta_t|\leq\eta_{\max}$. Notice that $\eta_{\max}$ here does not have to be a tight bound. Further, unlike confidence bounds that only contain $\theta_*$ with high probability \cite{abbasi2011improved, kazerouni2017conservative, chu2011contextual}, the uncertainty set $\Theta_t$ constructed by SME includes the true parameter vector $\theta_*$ almost surely.

\section{Algorithm Design \& Regret Guarantee}\label{sec: alg}
This section introduces our algorithm SME-OFU in Section \ref{subsec: algo design} and establishes its regret guarantee in Section \ref{subsec: regret theorem}.

\subsection{Our Algorithm Design: SME-OFU}\label{subsec: algo design}
Our algorithm design adopts the OFU principle to address the exploration-exploitation tradeoff.  The size of uncertainty sets is crucial for OFU-based algorithms \cite{abbasi2011improved,flynn2023improved}. Recent results have revealed promising numerical and theoretical performance of SME for uncertainty quantification under bounded noise \cite{zeng2025system,li2024learning}. Therefore, we leverage SME to construct the uncertainty sets for OFU. The pseudo-code is  provided in Algorithm \ref{alg:SME-OFUL}.

\begin{algorithm}[htp]
\caption{SME-OFU}\label{alg:SME-OFUL}
\textbf{Inputs:} $\Theta_0$, $\eta_{\max}$.\\
\For{$t=1,2,\cdots, T$}{
Receive context information $X_{t,1},\cdots,X_{t,K}$.\\
Select $a_t\in\{1,\dots, K\}$ by
$$(\theta_t,a_t) \gets \argmax_{\theta\in\Theta_{t-1}, a\in[K]} \theta^\top X_{t,a}.$$\\
Observe reward $r_{t,a_t}$.\\
Update the uncertainty set by SME: 
$$\Theta_t\gets \Theta_{t-1}\cap\left\{\hat{\theta}\in\mathbb{R}^d\ :\ |r_{t,a_t} - \hat{\theta}^\top X_{t,a_t}| \leq \eta_{\max}\right\}.$$
}
\end{algorithm}

Notice that SME-OFU only relies on a valid noise upper bound $\eta_{\max}$ and does not require $\eta_{\max}$ to be tight. Further, the implementation of SME-OFU does not require knowledge of  horizon $T$ or noise-dependent quantities $q(\cdot),c_0,\epsilon_0$ in Assumption \ref{ass: boundary-visiting}. Finally, Algorithm \ref{alg:SME-OFUL} can be easily generalized to asymmetric noise bounds $[\eta_{\min},  \eta_{\max}]$ by updating SME uncertainty sets by: $\eta_{\min} \leq r_{t,a_t} - \hat{\theta}^\top X_{t,a_t} \leq \eta_{\max}$.

\begin{remark}
  Though  SME in \eqref{eq: membership set}  requires  intersecting  an increasing number of linear constraints and thus incurs  high computational cost, computationally efficient variants of SME have been developed in the literature \cite{lu2021robust,fogel1982value}. An interesting future direction  is to incorporate these fast SME variants into LCB  and establish corresponding regret guarantees.
\end{remark}

\subsection{Regret Guarantee}\label{subsec: regret theorem}
Next, we establish the  regret guarantee of  SME-OFU. 
\begin{theorem}
    [Regret Bound]\label{thm: reg bound}
    Under Assumptions \ref{ass: noise} and \ref{ass: boundary-visiting}, for any  $\delta\in(0,1)$ and $d\geq 2$, with probability at least $1-\delta$, we have
    \begin{equation}\label{eq: reg bound}
        \textup{Reg}(T) \leq C_1 d^3\log\left(C_2 d^\frac{3}{2}T\right) + C_3 d\log\left(\frac{1}{\delta} \right)+ 1,
    \end{equation}
    where $C_1 = \frac{256}{c_0\epsilon_0}(\epsilon_0 + \theta_{\max}x_{\max}+1)$, $C_2 = 16x_{\max}\theta_{\max}+1$,  $C_3 = \frac{C_1}{12}$, and $\epsilon_0,c_0$ are defined in Assumption \ref{ass: boundary-visiting}.
\end{theorem}

The complete proof is in Section \ref{sec: proof sketch} and the Appendix. Theorem \ref{thm: reg bound} establishes a sublinear regret bound $O(d^3 \log(d^{3/2} T))$. 
More discussion is provided below.

\vspace{4pt}

\noindent\textbf{On Dependence of $T$.} Our regret bound of SME-OFU scales as $O(\log T)$ in terms of horizon  $T$. To compare with  existing regret bounds, it is worth noting that there are two types of regret bounds in LCB literature: instance-independent regret bounds, which typically scale as  $\tilde O(\sqrt T)$ \cite{abbasi2011improved,chu2011contextual,dani2008stochastic,guo2019adalinucb}, and instance-dependent regret bounds, which  typically scale as $O(\log^2(T)/\Delta)$, and  $\Delta$ represents a uniform lower bound on the  expected reward gap between the best arm and second best arm across all stages \cite{abbasi2011improved}. The instance-dependent regret bound $O(\log^2(T)/\Delta)$ is only meaningful when $\Delta>0$, which can be restrictive if context vectors $X_{t,a}$ change arbitrarily across stages.  In contrast,  Theorem \ref{thm: reg bound} does not assume $\Delta>0$ yet still achieves $\log(T)$ regret. That being said, our regret bound may be viewed as a different type of instance-dependent bound because Theorem \ref{thm: reg bound} relies on the boundary-visiting Assumption \ref{ass: boundary-visiting}. Therefore, our theorem suggests a potential advantage of SME-OFU when $\Delta$ is close to 0 while the reward noise follows a truncated Gaussian  or another distribution satisfying Assumption \ref{ass: boundary-visiting}.

\vspace{4pt}
\noindent\textbf{On Dependence of $d$.} The regret bound of SME-OFU scales as $\tilde O(d^3)$ in terms of the parameter dimension $d$ and does not depend on the number of arms $K$. The dependence on $d$ is worse than that of standard confidence-bound-based  OFU algorithms, whose regret bound typically scales as $\tilde O(d\sqrt T)$ \cite{dani2008stochastic,abbasi2011improved}. The worse dependence on $d$ may arise from   the geometric complexity of SME uncertainty sets. It is also possible, however, that the dependence on $d$ could be improved through a more refined analysis. Indeed, simulation results in \cite{li2024learning} show that the empirical performance of  SME and LSE's confidence bounds does not differ substantially as $d$ increases for linear system identification. This observation may suggest some looseness in the dimension dependence of our  analysis. Establishing a sharper dependence on $d$ is left for future work.

\begin{figure}[htp!]
\centering
\includegraphics[width=0.95\columnwidth]{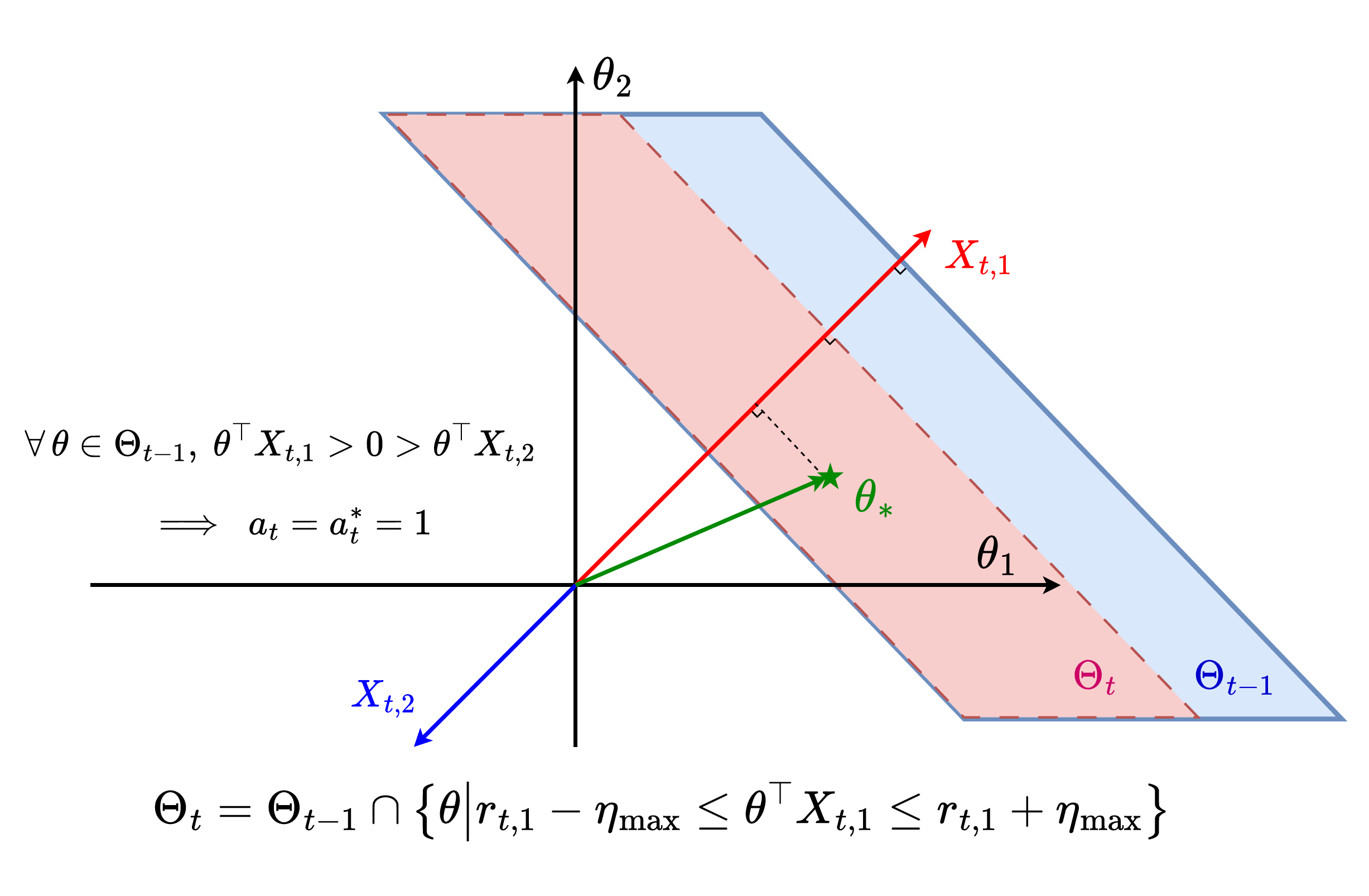}
\caption{An LCB example where the PE condition fails yet SME-OFU can select the optimal arm. Consider $d=2$. The context vectors $X_{t,1}, X_{t,2}$ lie in the same one-dimensional subspace for all $t\geq 1$. This violates the PE condition, so the uncertainty set $\Theta_{t-1}$ induced by SME does not shrink in the orthogonal direction of $X_{t,1}, X_{t,2}$. Nevertheless, the large uncertainty in the orthogonal direction does not affect the arm selection of SME-OFU because it does not affect the expected reward $\theta^\top X_{t,a}$. As a result, SME-OFU still selects the optimal arm $a_{t}=a_t^*$ despite the orthogonal uncertainty. }
\label{fig: PE example}
\end{figure}

\vspace{4pt}
\noindent\textbf{Relation with SME's Convergence Rate.} For linear system identification under i.i.d. noises, SME has been shown to enjoy a convergence rate of $\tilde O(1/T)$  under Assumptions \ref{ass: noise} and \ref{ass: boundary-visiting} \cite{zeng2025system,li2024learning}. Therefore, it is not surprising to attain an $\tilde O(1)$ regret bound for the bandit setting.\footnote{In many cases, the regret of an online algorithm roughly scales as the summation of its convergence rate or estimation error bounds \cite{abbasi2011improved,chu2011contextual}.} 
However, formally establishing the regret bound of SME-OFU is much more involved than directly applying SME's convergence rate. This is because the convergence-rate analysis of SME relies crucially on the persistent excitation (PE) condition, which requires the observed covariates to sufficiently explore all directions of the parameter space over at least a linear fraction of the time horizon. This condition is generally unavailable in LCB because the context vectors are generated by the environment and the agent can only choose among the available arms.

Our major technical contribution is  establishing a sublinear regret for SME-OFU without assuming PE. The key observation is that regret minimization does not require exact identification of the unknown parameter: it suffices to resolve only the uncertainty that can affect arm selection. Figure~\ref{fig: PE example} illustrates this intuition. Here, the context vectors $X_{t,1}$ and $X_{t,2}$  span a one-dimensional subspace, so PE does not hold and thus SME cannot reduce parameter uncertainty in the orthogonal direction. However, this unresolved uncertainty direction does not affect which arm is optimal. Once all parameters in the SME uncertainty set induce the same optimal arm,  further uncertainty reduction is unnecessary. This indicates that a sublinear regret may still be possible without the  PE condition.

To formalize this intuition, we develop a novel proof framework based on convex geometry. Specifically, we employ the volume of the minimum-volume enclosing ellipsoid (MVEE) of the SME uncertainty set as a potential function and characterize the imbalanced evolution of the uncertainty sets by geometric properties such as weighted widths and deep cuts. More proof details are discussed in Section \ref{sec: proof sketch}.

\vspace{4pt}

\noindent\textbf{Comparison with Existing Lower Bounds.} It is worth mentioning that our $O(\log T)$ regret upper bound does not violate the $\tilde \Omega(\sqrt T)$ lower bounds for Bernoulli bandits \cite{chu2011contextual,auer2002nonstochastic}. This is because we  additionally assume that the tight noise bounds are known in Assumption \ref{ass: boundary-visiting}. Though this assumption is compatible with truncated Gaussian and uniform distributions, it will induce a trivial  Bernoulli bandit problem. To see this, consider $r_{t,a}\sim\mathrm{Bernoulli}(p_a)$. Then
$
\eta_t=r_{t,a}-\mathbb E[r_{t,a}] \in \{-p_a,1-p_a\}.
$ Consequently, the tight noise bounds reveal $p_a$ exactly, making the optimal arm known and the regret trivially zero.\footnote{In Bernoulli bandits, the noise support is asymmetric and  arm dependent, while this paper assumes  symmetric and arm independent noise support for simplicity. This assumption can be relaxed. The most restrictive assumption here is knowing a tight noise support.}

\vspace{4pt}

\noindent\textbf{On the known tight bound.} Algorithm \ref{alg:SME-OFUL} is still well-defined under a conservative bound $\eta_{\max}$, but Theorem \ref{thm: reg bound} requires a tight noise bound $\eta_{\max}$ to establish the regret guarantee. Indeed, assuming knowledge of a tight bound on reward noise is  quite restrictive. However, there are variants of SME that can learn the noise bound and reduce parameter uncertainty at the same time and still enjoy convergence rate guarantees \cite{li2024learning}. It is our ongoing work to apply these SME variants to bandit settings with unknown noise bounds. The main technical challenge is again the lack of the PE condition in the bandit settings.  The analysis developed in this paper may provide a foundation for establishing regret guarantees for these extensions without assuming PE.
\section{Proof of Theorem \ref{thm: reg bound}}\label{sec: proof sketch}

In this section, we provide a proof sketch of Theorem \ref{thm: reg bound}. The proof of the supporting lemmas are in the Appendix. This section, combined with the Appendix, completes the proof of Theorem \ref{thm: reg bound}. The proof follows the four steps below.
\begin{enumerate}
    \item[1)] Bound  regret by ellipsoidal widths. 
    \item[2)] Bound  the ellipsoidal widths with a deep cut counter. 
    \item[3)] Bound the deep cut counter by convex geometry. 
    \item[4)] Derive the regret bound based on the above three steps.
\end{enumerate}

\noindent\textbf{Step 1: Weighted ellipsoidal width reduction.} First, we introduce an ellipsoid-based characterization of the size of the membership set \eqref{eq: membership set}.

For convenience of analysis, $\forall\,t\geq 0$, define the error set
\begin{equation}
    \label{eq: error set}
    \Gamma_t = \Theta_t - \theta_*.
\end{equation}
As a translation of $\Theta_t$, $\Gamma_t$ preserves the size of the SME uncertainty set. By the SME updating rule \eqref{eq: membership set}, the error sets $\{\Gamma_t\}_{t\geq 0}$ can equivalently be represented by 
\begin{align*}
    & \Gamma_0 = [-\theta_{\max},\theta_{\max}]^d-\theta_*\subseteq [-2\theta_{\max}, 2\theta_{\max}]^d,\\
    & \Gamma_t = \Gamma_{t-1}\cap\{\gamma\ :\ -b_t\leq X_{t,a_t}^\top \gamma \leq b'_t\},\ \forall\,t\geq 1,
\end{align*}
where $b_t = \eta_{\max} - \eta_t$ and $\ b'_t = \eta_{\max}+\eta_t$. The error set $\Gamma_t$ encodes all possible deviations of the membership set $\Theta_t$ from the ground-truth parameter $\theta_*$. By the  SME updating rule \eqref{eq: membership set} and the linear reward model \eqref{eq: cb reward}, the membership sets and the error sets are successive intersections of half-spaces, whose sizes are difficult to track. To address this difficulty, we introduce the notion of enclosing ellipsoids.

\begin{definition}
    [Enclosing ellipsoid]\label{def: MVEE}
    For a convex and compact set $K\subseteq\mathbb{R}^d$ with a non-empty interior and an ellipsoid $\mathcal{E}(B,c)=\{y\in\mathbb{R}^d\ :\ (y-c)^\top B^{-1}(y-c)\leq 1\}$ where $B\succ 0$, we say $\mathcal{E}(B,c)$ encloses $K$ if $K\subseteq \mathcal{E}(B,c)$.
\end{definition}
Next, we define the minimum volume enclosing ellipsoid (MVEE) of a convex body.
\begin{proposition}[{Existence and uniqueness of MVEE
\cite[Theorems~2.2 and~2.7]{GulerGurtuna2007},\cite{Gruber2011}}]\label{prop: mvee}
    Let $K \subseteq \mathbb{R}^d$ be a convex and compact set with a non-empty interior. There exists a unique ellipsoid of minimum volume enclosing $K$.
\end{proposition}

The notion of MVEE eases our analysis of the size of the error set (equivalently, that of the membership set). To guarantee the existence and uniqueness of the MVEE, for any pre-selected $\mu >0$, define the regularized error set
\begin{equation}
    \label{eq: regularized error set}
    K_t^\mu := \Gamma_t + \mu B_2^d = \left\{\gamma + \mu y\ :\ \gamma\in\Gamma_t,\ y\in B_2^d\right\},
\end{equation}
where $B_2^d$ denotes the $\ell_2$ unit ball in $\mathbb{R}^d$. The $\mu$-regularization guarantees that $K_t^\mu$ has a non-empty interior even in the situation where the membership set's uncertainty in some subspaces is completely eliminated, and thus satisfies the conditions of Proposition \ref{prop: mvee}, making its MVEE well-defined.

Since $\theta_*\in\Theta_t$ for any $t\geq 0$, we have $0\in\Gamma_t$ and $\mu B_2^d\subseteq K_t^\mu$. In particular, $K_t^\mu$ is compact, convex, and has a non-empty interior for every $t\geq 0$.

Let $$E_t \!=\!\mathcal{E}(B_t,c_t)\!:=\!\left\{\gamma\in\mathbb{R}^d\!:\!(\!\gamma-c_t\!)^\top\!B_{t}^{-1}\!(\!\gamma\!-\!c_t\!)\!\leq\!1\!\right\}$$
be the MVEE of $K_t^\mu$, where $B_t\succ 0$. The $X_{t,a_t}$-weighted projection width of $E_{t-1}$ onto the vector $X_{t,a_t}$ is $$W_t := 2\sqrt{X_{t,a_t}^\top B_{t-1}X_{t,a_t}}.$$ The weighted width $W_t$ quantifies the uncertainty of $\Theta_{t-1}$ in the direction of $X_{t,a_t}$, which dominates the regret induced by $X_{t,a_t}$. Formally, we have the following lemma.
\begin{lemma}[Reduce regret to weighted width]\label{lm: reg to width}
    $\forall\,T\geq 1$, $\textrm{Reg}(T)\leq \sum_{t=1}^T W_t$.
\end{lemma}
\noindent The proof of Lemma \ref{lm: reg to width} is in Appendix \ref{app: proof of lm 1}.

\noindent\textbf{Step 2: Bound ellipsoidal widths by high-probability deep-cut counter.}
Lemma \ref{lm: reg to width} suggests that every valid upper bound on $\sum_{t=1}^TW_t$ is a valid upper bound on $\textrm{Reg}(T)$. To upper bound $\sum_{t=1}^T W_t$, we partition $\{W_t\}_{t=1}^T$ into two groups by a threshold $w_\mu$ (to be formally defined in Lemma \ref{lm: freedman}): namely,
\begin{equation}
    \label{eq: partition}
    \sum_{t=1}^T W_t = \sum_{t:W_t< w_\mu} W_t + \sum_{t:\,W_t\geq w_\mu}W_t.
\end{equation}
Next, we aim to upper bound the two groups of widths respectively: for the smaller widths, we simply bound them by the threshold $w_\mu$ per se; likewise, to bound the larger widths, we require the notion of deep cut events and a deep cut counter. Intuitively, a deep cut refers to an informative update to the SME \eqref{eq: membership set}. The formal definition of deep cuts is as follows.
\begin{definition}[Deep cut events and indicator]\label{def: deep cut}
   For each $t\geq 1$, we define
$$\alpha_t \coloneq X_{t,a_t}^\top c_{t-1},\quad \rho_t \coloneq \sqrt{X_{t,a_t}^\top B_{t-1}X_{t,a_t}} = \frac{W_t}{2}.$$
When $\rho_t >0$, the normalized, one-sided offsets of the \textup{regularized} SME cut at step $t$ are
$$q_t^+ \coloneq \frac{b'_t +\mu x_{max} - \alpha_t}{\rho_t}, \quad \ q_t^{-} \coloneq \frac{b_t + \mu x_{max} +  \alpha_t}{\rho_t}.$$ For any $\xi\in(0,1)$, we say a $\xi$-deep cut event happens if either $q_t^+\leq \frac{\xi}{d}$ or $q_t^-\leq\frac{\xi}{d}$ holds. Define the $\xi$-deep cut indicator 
\begin{align*}
    I_t(\xi)\coloneq \begin{cases}
        \mathds{1}\left\{q_t^+ \leq \frac{\xi}{d}\ \textrm{or}\ q_t^-\leq \frac{\xi}{d}\right\} & \text{ if $\rho_t >0$,}\\
        0 & \text{ if $\rho_t=0$.}
    \end{cases}
\end{align*}
\end{definition}
\begin{remark}
    Notice that $\rho_t=0$ if and only if $X_{t,a_t}=0$, which implies $W_t = 0$. The SME constraint is then satisfied by all parameters in $\Theta_{t-1}$. In this case, $\Theta_t = \Theta_{t-1}$ and $E_t=E_{t-1}$. Hence, the convention that $I_t(\xi) = 0$ when $\rho_t=0$ is consistent with the intuition of deep cuts.
\end{remark}
We state the following lemma, which provides a high-probability upper bound on $\sum_{t=1}^TW_t$.

\begin{lemma}
    [High-probability upper bound on $\sum_{t=1}^TW_t$]\label{lm: freedman}
    For any $\delta\in (0,1)$, with probability no less than $1-\delta$,
    \begin{align*}
    &\sum_{t=1}^T W_t\!\leq\!Tw_\mu\!+\!\frac{4d(w_\xi + W^\mu_{\max})}{c_0\xi w_\xi}\!\left(2\sum_{t=1}^T I_t(\xi) + \frac{8}{3}\log\frac{1}{\delta}\right)\!,
\end{align*}
where $w_\mu = \frac{4d\mu x_{max}}{\xi}$, $w_\xi = \frac{4d\epsilon_0}{\xi}$, and $W_{\max}^\mu = 2x_{max}\sqrt{d}(2\theta_{\max}\sqrt{d}+\mu)$.
\end{lemma}

\noindent The proof of Lemma \ref{lm: freedman} is in Appendix \ref{app: proof of lm 2}.

\noindent\textbf{Step 3: Upper bound on deep-cut counter.}
We define the volumetric potential of $E_t$ by $\Phi_t = \log\vol(B_2^d) + \frac{1}{2}\log\det B_t$. We have the following lemma discussing the evolution of the volumetric potential.
\begin{lemma}[Bound deep cuts with volumetric potential]\label{lm: informal deep cut vol potential}
    For any fixed $\xi\in(0,1)$, $\forall\,t\geq 1$, $\Phi_{t-1} - \Phi_t \geq \Omega_{d,\xi}I_t(\xi)$,
    where $\Omega_{d,\xi} := -\log\left[\frac{d+\xi}{d+1}\left(\frac{d^2-\xi^2}{d^2-1}\right)^{\frac{d-1}{2}}\right] \geq \frac{(1-\xi)^2}{2(d+1)}.$ Furthermore,
    \begin{equation}\label{eq: lm 3}
        \sum_{t=1}^T I_t(\xi) \leq \frac{d}{\Omega_{d,\xi}}\log\left(\frac{2\theta_{\max}\sqrt{d} + \mu}{\mu}\right).
    \end{equation}
\end{lemma}
\noindent The proof of Lemma \ref{lm: informal deep cut vol potential} is in Appendix 
\ref{app: proof of lm 3}.

\noindent\textbf{Step 4: The final assembly.}
By Lemma \ref{lm: reg to width},
\begin{equation}
    \label{eq: step 1}
   \textup{Reg}(T) \leq \sum_{t=1}^T W_t.
\end{equation}
By Lemma \ref{lm: freedman}, $\forall\,\delta\in(0,1)$, with probability no less than $1-\delta$,
\begin{align}
    \label{eq: step 2}
    \sum_{t=1}^T\!W_t\!
    \leq\! Tw_\mu\!+\!\frac{4d(w_\xi + W^\mu_{\max})}{c_0\xi w_\xi}\!\left(\!2\sum_{t=1}^T I_t(\xi) + \frac{8}{3}\log\frac{1}{\delta}\!\right)\! .
\end{align}
By Lemma \ref{lm: informal deep cut vol potential},
\begin{equation}
    \label{eq: step 3}
    \sum_{t=1}^T I_t(\xi) \leq \frac{d}{\Omega_{d,\xi}}\log\left(\frac{2\theta_{\max}\sqrt{d} + \mu}{\mu}\right).
\end{equation}
Combining inequalities \eqref{eq: step 1}, \eqref{eq: step 2}, and \eqref{eq: step 3}, and using $\Omega_{d,\xi}\geq \frac{(1-\xi)^2}{2(d+1)}$, we obtain that
\begin{align*}
    &\textup{Reg}(T)\!\leq\!\frac{4d\mu x_{\max}T}{\xi}\!+\!\left[\frac{\xi x_{\max}\sqrt{d}(2\theta_{\max}\sqrt{d}+\mu)}{2d\epsilon_0}\!+\!1\right]\times\\
    &\left[\frac{16d^2(d+1)}{c_0\xi(1-\xi)^2}\log\left(\frac{2\theta_{\max}\sqrt{d}+\mu}{\mu}\right)+\frac{32d}{3c_0\xi}\log\frac{1}{\delta}\right]
\end{align*}
holds with probability no less than $1-\delta$. Next, choose $\xi=\frac{1}{2}$ and $\mu=\frac{1}{8dx_{\max}T}$. Then
\begin{align*}
    &\textup{Reg}(T)\!\leq\!1\!+\!\left[\frac{\theta_{\max}x_{\max}}{2\epsilon_0}+\frac{1}{32d^{\frac{3}{2}}\epsilon_0 T} + 1\right]\times\\
    &\left[\frac{128d^2(d+1)}{c_0}\log\left(16\theta_{\max}x_{\max}d^{\frac{3}{2}}T + 1\right) + \frac{64d}{3c_0}\log\frac{1}{\delta}\right].
\end{align*}
For $T\geq 1$ and $d\geq 2$, we have $\frac{1}{32d^{\frac{3}{2}}T}\leq 1$. Therefore,
\begin{equation*}
    \frac{\theta_{\max}x_{\max}}{2\epsilon_0}+\frac{1}{32d^{\frac{3}{2}}\epsilon_0 T} + 1 \leq \frac{\epsilon_0 + \theta_{\max}x_{\max}+1}{\epsilon_0}.
\end{equation*}
Furthermore, using that $d+1\leq 2d$ and that $d^{\frac{3}{2}}T\geq 1$, we obtain that
\begin{align*}
    &\frac{128d^2(d+1)}{c_0} \leq \frac{256d^3}{c_0},\\
    &16\theta_{\max}x_{\max}d^{\frac{3}{2}}T + 1 \leq (16\theta_{\max}x_{\max}+1)d^{\frac{3}{2}}T.
\end{align*}
Writing
\begin{align*}
    &C_1=\frac{256}{c_0\epsilon_0} ( \epsilon_0 + \theta_{\max}x_{\max} + 1 ),\\
    &C_2 = 16\theta_{\max}x_{\max} + 1,\\
    &C_3=\frac{C_1}{12},
\end{align*}
we can conclude that with probability no less than $1-\delta$,
\begin{equation*}
    \textup{Reg}(T)\leq C_1 d^3\log\left(C_2d^{\frac{3}{2}}T\right) + C_3d\log\frac{1}{\delta}+1.
\end{equation*}
This completes the proof of Theorem \ref{thm: reg bound}.
\section{Numerical Experiments}\label{sec: num}

This section provides a numerical comparison of   SME-OFU with existing  algorithms OFUL  \cite{abbasi2011improved} and LinUCB \cite{chu2011contextual}.

The experiment settings are discussed below. Our numerical experiment uses a synthetic setting inspired by \cite{guo2019adalinucb}. In particular, we let the  ground-truth parameter $\theta_*$ be uniformly randomly generated from $[-1,1]^5$. The contexts $\{X_{t,1},\cdots,X_{t,10}\}_{t\geq 1}$ are i.i.d. and uniformly sampled from the $\ell_2$-ball with radius $x_{\max}=1$. We simulate two types of bounded i.i.d. random noise $\{\eta_t\}_{t\geq 1}$: (1) uniform noise on $[-1,1]$, and (2) $\mathcal{N}(0,1)$ noise truncated on $[-1,1]$. Once generated,  $\theta_*$ and $\{X_{t,1},\cdots,X_{t,10}\}_{t=1}^T$ are fixed for each experiment. For  OFUL \cite{abbasi2011improved}, we set $\lambda = 0.1$, $\delta = 0.05$, $L=1$, $R_{\textup{uniform}}=\frac{1}{\sqrt{3}}$, and $R_{\textup{truncGaussian}}=\sqrt{1-\frac{2\varphi(1)}{2\Phi(1)-1}}\approx 0.54$\footnote{$\varphi$ and $\Phi$ are respectively the standard Gaussian density and cumulative distribution functions.}. For  LinUCB, we use Algorithm 1 in \cite{chu2011contextual} with $\alpha=2$. The shaded area in Figure \ref{fig: synthetic} indicates  empirical standard deviations across $10$ trials. 

\begin{figure}[!t]
    \centering
    \begin{subfigure}[t]{0.47\columnwidth}
                \centering

\includegraphics[height=1.in]{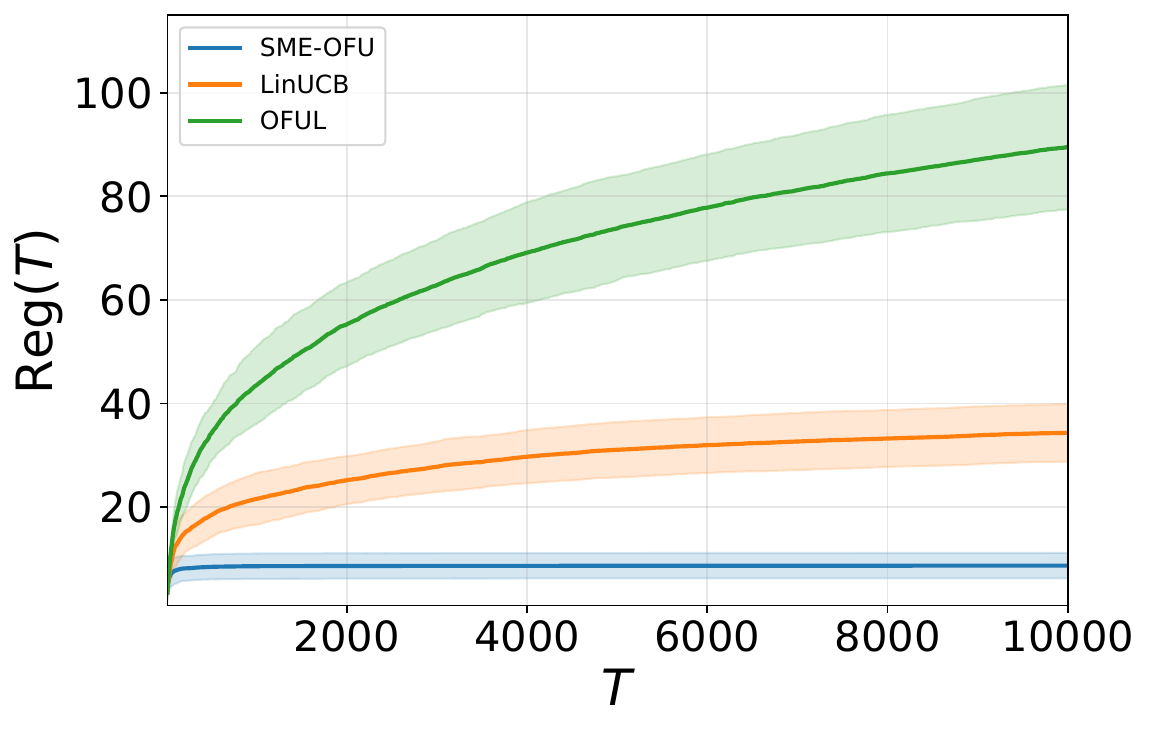}
        \caption{Uniform noise.}
        \label{subplot: uniform}
    \end{subfigure}%
\hfill
 \begin{subfigure}[t]{0.47\columnwidth}
                \centering

\includegraphics[height=1.in]{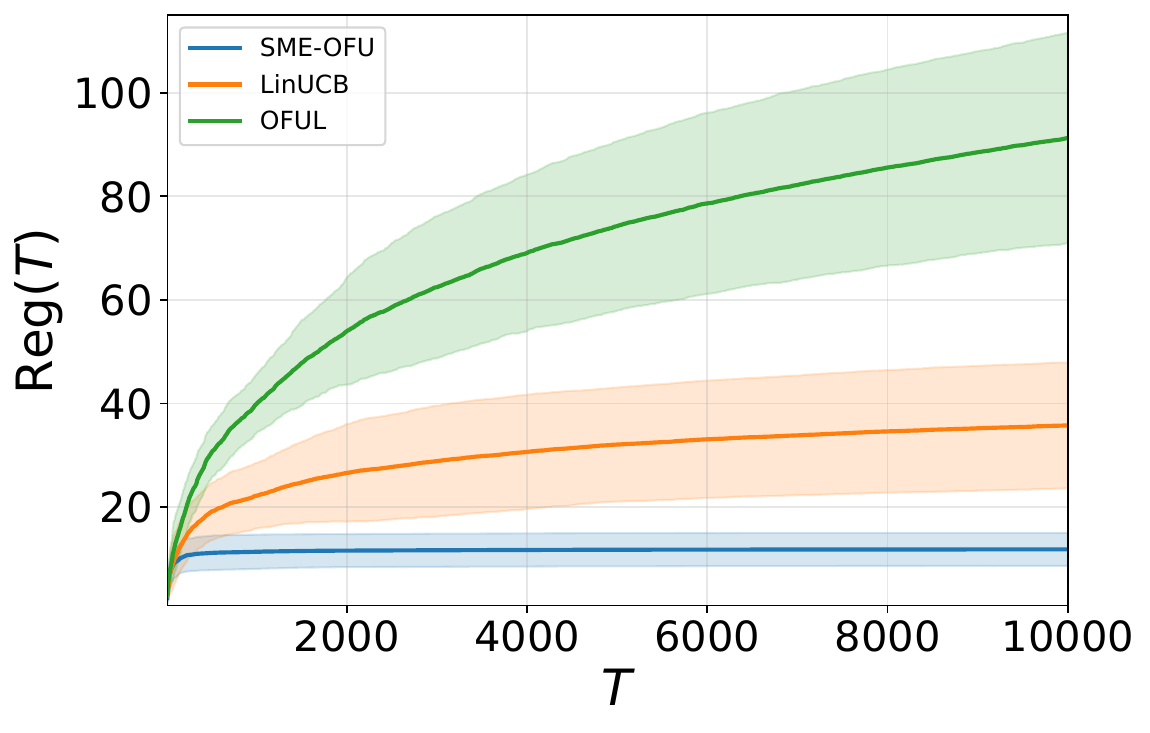}
        \caption{Truncated Gaussian noise.}
        \label{subplot: truncated gaussian}
    \end{subfigure}
    \caption{The  regrets  of SME-OFU,  OFUL \cite{abbasi2011improved}, and LinUCB \cite{li2010contextual} under uniform reward noise (Figure \ref{fig: synthetic}(\subref{subplot: uniform})) and truncated Gaussian reward noise (Figure \ref{fig: synthetic}(\subref{subplot: truncated gaussian})).} 
    \label{fig: synthetic}
\end{figure}

Figure \ref{fig: synthetic} plots the regret  of SME-OFU, OFUL and LinUCB.
It can be observed that, under both uniform noise and truncated Gaussian noise,  the  regret $\text{Reg}(T)$ of SME-OFU is much smaller than the regrets of OFUL and LinUCB after a short exploration phase, demonstrating promising potential of SME-based uncertainty quantification for online bandit learning. Further, the numerical regret performance is  consistent with the $O(\log T)$ regret bound of SME-OFU in comparison  with $\tilde O(\sqrt T)$ regret bounds of OFUL and LinUCB. Finally, SME-OFU  enjoys a smaller standard deviation compared with OFUL and LinUCB. This is probably because  SME-OFU uses almost-sure uncertainty sets constructed by SME, while OFUL  and LinUCB utilize high-probability confidence regions, which   involve more randomness.
\section{Conclusion}
This paper introduces set membership estimation to  linear contextual bandits by proposing an SME-OFU algorithm and establishing a sublinear regret guarantee under assumptions on bounded reward noises and known tight noise bounds. We also numerically test SME-OFU in comparison with existing linear contextual bandit algorithms. There are many interesting future directions: 1)  extending the algorithm to  unknown noise bounds by integrating  SME variants that learn the noise bound and reduce parameter uncertainty simultaneously; 2) improving the computational efficiency while preserving the regret guarantee for SME-based bandit algorithm design with fast SME variants; 3) applying SME-based uncertainty quantification to other bandit learning settings, such as adversarial contextual bandits, nonlinear contextual bandits, neural contextual bandits, and reinforcement learning.
\appendix
\textbf{List of Content.} The Appendix is organized as follows.
\begin{itemize}
    \item \textbf{Lemma \ref{lm: reg to width}.} Appendix \ref{app: proof of lm 1} provides the complete proof.
    \item \textbf{Lemma \ref{lm: informal deep cut vol potential}.} Appendix  \ref{app: proof of lm 3} provides the main argument containing Claim \ref{lm: vol int}, Corollary \ref{cor: deep cut volume}, and Claim \ref{lm: kappa deep cuts}, whose proofs are respectively provided in Appendices \ref{app: proof of claim 1}, \ref{app: proof of corollary 1}, and \ref{app: proof of claim 2}.
    \item \textbf{Lemma \ref{lm: freedman}.} Appendix \ref{app: proof of lm 2} provides the main argument containing Claim \ref{lm: lb pt} and Proposition \ref{prop: terminal}, whose proofs are respectively provided in Appendices \ref{app: proof of claim 3} and \ref{app: proof of prop 1}.
\end{itemize}

\subsection{Proof of Lemma \ref{lm: reg to width}}\label{app: proof of lm 1}
    First, we use the OFU principle to show that the regret cannot exceed the support function in each round. Recall that
    \begin{align*}
    \textup{Reg}(T) &= \sum_{t=1}^T \theta_*^\top \left(X_{t,a_t^*} - X_{t,a_t}\right)\\
    &\leq \sum_{t=1}^T (\theta_t^\top X_{t,a_t} - \theta_*^\top X_{t,a_t})=\sum_{t=1}^T(\theta_t - \theta_*)^\top X_{t,a_t},
\end{align*}
where $a_t^* := \arg\max_{a\in[K]} \theta_*^\top X_{t,a}$, and $(\theta_t, a_t) := \arg\max_{\theta\in\Theta_{t-1}, a\in [K]} \theta^\top X_{t,a}$. 

Define the support function at round $t$ by
\begin{equation*}
    h(t) := \sup_{\gamma\in\Gamma_{t-1}}\gamma^\top X_{t,a_t}. 
\end{equation*}
Since $\theta_t-\theta_*\in\Gamma_{t-1}$, then $\forall\,t\geq 1$, $(\theta_t-\theta_*)^\top X_{t,a_t} \leq h(t)$. Therefore, $\textup{Reg}(T)\leq \sum_{t=1}^T h(t)$.

Second, we show that $\forall\,t\geq 1$, $h(t)\leq W_t$. Recall that $W_t := 2\sqrt{X_{t,a_t}^\top B_{t-1}X_{t,a_t}}$,
where $E_{t-1} = \mathcal{E}(B_{t-1},c_{t-1})$ is the MVEE of $K_{t-1}^\mu$. It follows that $$\Gamma_{t-1}\subseteq K_{t-1}^\mu \subseteq E_{t-1},$$ and that $$h(t) = \sup_{\gamma\in\Gamma_{t-1}}X_{t,a_t}^\top \gamma \leq \sup_{\gamma\in E_{t-1}}X_{t,a_t}^\top\gamma.$$
Notice that
\begin{align*}
    \sup_{\gamma\in E_{t-1}}X_{t,a_t}^\top\gamma &= \sup_{\gamma\in E_{t-1}}\left[X_{t,a_t}^\top (\gamma - c_{t-1}) + X_{t,a_t}^\top c_{t-1}\right]\\
    &=\sup_{\gamma\in E_{t-1}}\left[X_{t,a_t}^\top (\gamma - c_{t-1})\right] + X_{t,a_t}^\top c_{t-1}\\
    &=\sqrt{X_{t,a_t}^\top B_{t-1}X_{t,a_t}} + X_{t,a_t}^\top c_{t-1}.
\end{align*}
Similarly, $\inf_{\gamma\in E_{t-1}}X_{t,a_t}^\top\gamma \!= \!\!-\sqrt{X_{t,a_t}^\top B_{t-1}X_{t,a_t}}\! \!+\!X^\top_{t,a_t} c_{t-1}$.
Therefore, we have 
$$W_t = \sup_{\gamma\in E_{t-1}}X_{t,a_t}^\top\gamma - \inf_{\gamma\in E_{t-1}}X_{t,a_t}^\top\gamma.$$
Moreover, since $0\in \Gamma_{t-1}\subseteq K_{t-1}^\mu\subseteq E_{t-1}$, we have $\inf_{\gamma\in E_{t-1}}X_{t,a_t}^\top\gamma \leq 0$. It follows that 
\begin{align*}
    h(t) &\leq \sup_{\gamma\in E_{t-1}}X_{t,a_t}^\top\gamma\\& \leq \sup_{\gamma\in E_{t-1}}X_{t,a_t}^\top\gamma - \inf_{\gamma\in E_{t-1}}X_{t,a_t}^\top\gamma = W_t.
\end{align*}
\subsection{Proof of Lemma \ref{lm: informal deep cut vol potential}}\label{app: proof of lm 3}
    First define the MVEE operator: for any compact and convex set $S$ with a non-empty interior, let $M(S)$ denote its MVEE. For two such sets $A\subseteq B$, every ellipsoid enclosing $B$ also encloses $A$. Therefore, $\vol(M(A))\leq \vol(M(B))$.
    Consider the intersection of an ellipsoid $$E=\mathcal{E}(B,c) = \{\gamma\,:\,(\gamma-c)^\top B^{-1}(\gamma-c)\leq 1\}$$ and a half-space $$H=\{\gamma\,:\,x^\top\gamma\leq u\},\ x\neq 0.$$ For convenience, denote $r=\sqrt{x^\top Bx}$, $q=\frac{u-x^\top c}{r}$. We make the following claim.
    \begin{claim}[Volumetric potential change under intersections]\label{lm: vol int} If $q\in \left(-1,\frac{1}{d}\right)$, then $$\frac{\vol(M(E\cap H))}{\vol(E)}\leq R_d(q),$$
    where $R_d(q) := \frac{d(1+q)}{d+1}\left(\frac{d^2(1-q^2)}{d^2 - 1}\right)^{\frac{d-1}{2}}$. 
    Consequently, 
    $$\log\vol (E) - \log\vol (M(E\cap H)) \geq -\log R_d(q).$$
    \end{claim}
    Claim \ref{lm: vol int} is inspired by \cite[Theorem~6.1]{GulerGurtuna2007} and is proved in Appendix \ref{app: proof of claim 1}. It geometrically addresses the change of MVEE volumes under deep cuts, which can be further generalized by the following corollary.
    \begin{corollary}[Volumetric potential change under deep intersection]\label{cor: deep cut volume}
    For any fixed $\xi\in (0,1)$, if $-1 < q\leq \frac{\xi}{d}$,
    then $$\log\vol(E) - \log\vol(M(E\cap H)) \geq \Omega_{d,\xi},$$
    where $\Omega_{d,\xi} := -\log\left[\frac{d+\xi}{d+1}\left(\frac{d^2 - \xi^2}{d^2-1}\right)^{\frac{d-1}{2}}\right]\geq \frac{(1-\xi)^2}{2(d+1)}$.
    \end{corollary}
The proof is in Appendix \ref{app: proof of corollary 1}. Further, we have the following claim.
\begin{claim}
    [Potential change]\label{lm: kappa deep cuts}
    $\forall\,t\geq 1$, $\Phi_{t-1}-\Phi_t\geq\Omega_{d,\xi} I_t(\xi)$.
\end{claim}
The proof of Claim \ref{lm: kappa deep cuts} is in Appendix \ref{app: proof of claim 2}. From Claim \ref{lm: kappa deep cuts}, by taking the telescoping sum, we obtain $\sum_{t=1}^T I_t(\xi) \leq \frac{\Phi_0 - \Phi_T}{\Omega_{d,\xi}}$. Notice that $\Phi_0$ is fixed. Thus, we can bound $\sum I_t$ by bounding $\Phi_T$. Since $\mu B_2^d\subseteq K_t^\mu\subseteq E_t$, we have 
\begin{align*}
    \Phi_t &= \log\vol(E_t)\geq \log\vol(\mu B_2^d)\\
    &= \log(\mu^d\vol(B_2^d))
    = d\log\mu + \log\vol(B_2^d).
\end{align*}
Meanwhile, $ K_0^\mu = ([-\theta_{\max},\theta_{\max}]^d - \theta_*)+ \mu B_2^d \subseteq [-2\theta_{\max},2\theta_{\max}]^d+\mu B_2^d \subseteq (2\theta_{\max}\sqrt{d}+\mu)B_2^d$.
Therefore, $$\Phi_0\leq \log\vol(B_2^d) + d\log(2\theta_{\max}\sqrt{d}+\mu).$$
Accordingly, we have $\Phi_0 - \Phi_T \leq d\log\left(\frac{2\theta_{\max}\sqrt{d}+\mu}{\mu}\right)$.

\subsection{Proof of Claim \ref{lm: vol int}}\label{app: proof of claim 1}
First, we consider a deep cut onto a unit ball. Let $$K(q) := B_2^d \cap \{y\in\mathbb{R}^d\ :\ y_1 \leq q\},$$
where $q\in \left(-1,\frac{1}{d}\right)$. Define
\begin{align*}
    &m(q) := \frac{dq-1}{d+1},\\
    &a(q) := \frac{d(1+q)}{d+1},\\
    &b(q) := \sqrt{\frac{d^2(1-q^2)}{d^2 - 1}}.
\end{align*}
Let $$\mathcal{E}(q)\! :=\! \left\{\!(y_1,y_\perp)\!\in\!\mathbb{R}\!\times\!\mathbb{R}^{d-1}\!:\! \frac{[y_1 - m(q)]^2}{a(q)^2}\!+\!\frac{||y_\perp||_2^2}{b(q)^2} \leq 1\!\right\}.$$
We first show that
\begin{equation*}
    \forall\,q\in \left(-1,\frac{1}{d}\right),\ K(q)\subseteq \mathcal{E}(q).
\end{equation*}
To show this, notice that for any $y= (s,z)\in K(q)$, where $s\in\mathbb{R}$, and $z\in\mathbb{R}^{d-1}$, we have $$-1\leq s\leq q,\ ||z||_2^2 \leq 1-s^2.$$
    Hence, it suffices to prove that $$\forall\, s\in [-1,q],\ \frac{[s-m(q)]^2}{a(q)^2} + \frac{1-s^2}{b(q)^2} \leq 1.$$
    To prove this, define $$F(q,s) := \frac{[s-m(q)]^2}{a(q)^2} + \frac{1-s^2}{b(q)^2}.$$
    Taking difference yields $$1 - F(q,s) = \frac{2(d+1)(q-s)(1+s)(1-dq)}{d^2(1-q)(1+q)^2}.$$
    Notice that the denominator is strictly positive, and every factor of the numerator is non-negative. Therefore, $1-F(q,s) \geq 0$.
    Hence, $\forall\ y \in K(q)$, we can deduce that $y\in\mathcal{E}(q)$. It follows that $K(q)\subseteq \mathcal{E}(q)$.
This provides an upper bound on the volume of the MVEE of $K(q)$. We further show the following: let $E^*(q)$ be the MVEE of $K(q)$. Then $$\frac{\vol(E^*(q))}{\vol(B^d_2)} \leq R_d(q),$$
where $R_d(q) = \frac{d(1+q)}{d+1}\left(\frac{d^2(1-q^2)}{d^2 - 1}\right)^{\frac{d-1}{2}}$.
To obtain this, recall that $E^*(q)$ is the MVEE of $K(q)$, and $\mathcal{E}(q)$ is also an ellipsoid enclosing $K(q)$. Thus,
    \begin{align*}
        &\vol(E^*(q)) \leq \vol(\mathcal{E}(q)) = a(q)\cdot (b(q))^{d-1}\vol(B^d_2)\\
        =& \frac{d(1+q)}{d+1}\left(\frac{d^2(1-q^2)}{d^2 - 1}\right)^{\frac{d-1}{2}}\vol(B^d_2) = R_d(q)\vol(B^d_2).
    \end{align*}
Hence, we extend the volumetric ratio bound to a more general ellipsoid. For the ellipsoid $\mathcal{E}(B,c)$, define an affine map $$y = T(\gamma) := B^{-\frac{1}{2}}(\gamma - c).$$ Notice that $T(E) = B_2^d$. Under this affine map, we have that  $x^\top\gamma\leq u $ is equivalent to $\left(B^{\frac{1}{2}}x\right)^\top y\leq u - x^\top c$.
Denote  $\bar x := \frac{B^{\frac{1}{2}}x}{||B^{\frac{1}{2}}x||_2} = \frac{B^{\frac{1}{2}}x}{r}$.
Then $T(\cdot)$ maps $H$ into
$H' := T(H) = \{y\ :\ \bar{x}^\top y \leq q\}$. Choose an orthogonal matrix $Q$ satisfying $Q\bar{x}=e_1$, and define
\begin{equation*}
    S(\gamma):=QT(\gamma) = QB^{-\frac{1}{2}}(\gamma-c).
\end{equation*}
Then $S(E) = B_2^d$ and $S(E\cap H) = K(q)$. Since $K(q)\subseteq \mathcal{E}(q)$, then the ellipsoid $S^{-1}(\mathcal{E}(q))$ encloses $E\cap H$, where the inverse map is $S^{-1}(y) = c+B^{\frac{1}{2}}Q^\top y$, whose volume scaling factor is $\sqrt{\det(B)}$. Therefore,
\begin{align*}
    &\frac{\vol(M(E\cap H))}{\vol(E)} \leq \frac{\vol(S^{-1}(\mathcal{E}(q)))}{\vol(E)}\\
    =&\frac{\sqrt{\det(B)}\vol(\mathcal{E}(q))}{\sqrt{\det(B)}\vol(B_2^d)}\leq R_d(q).
\end{align*}
The logarithmic statement follows immediately and completes the proof of Claim \ref{lm: vol int}.

\subsection{Proof of Corollary \ref{cor: deep cut volume}}\label{app: proof of corollary 1}
Recall that by Claim \ref{lm: vol int}, $\log\vol(E) - \log\vol(M(E\cap H)) \geq -\log(R_d(q))$.
Notice that 
\begin{equation*}
    \frac{d}{dq}\log R_d(q) = \frac{1-dq}{1-q^2}>0,\ \forall\,q\in\left(-1,\frac{1}{d}\right).
\end{equation*}
Thus, $R_d(q)$ is monotonically increasing on $\left(-1,\frac{1}{d}\right)$. Thus, $-\log(R_d(q))$ is monotonically decreasing on this interval. Given that $q\leq \frac{\xi}{d}$,
we have $$\log\vol(E) - \log\vol(M(E\cap H)) \geq \Omega_{d,\xi}\!=\!-\log\left(\!R_d\left(\!\frac{\xi}{d}\!\right)\!\right).$$
Furthermore, notice that $$\Omega_{d,\xi} = \log\frac{d+1}{d+\xi} - \frac{d-1}{2}\log\frac{d^2-\xi^2}{d^2-1}.$$
By the fact that $x+\log(1-x)\leq 0$ on $x\in[0,1)$, we have $$\log\frac{d+1}{d+\xi} = -\log\left(1 - \frac{1-\xi}{d+1}\right)\geq \frac{1-\xi}{d+1}.$$
By the fact that $\log(1+x)\leq x$ for all $x>-1$, we have $$\log\frac{d^2-\xi^2}{d^2-1} = \log\left(1 + \frac{1 - \xi^2}{d^2 - 1}\right) \leq \frac{1-\xi^2}{d^2 - 1}.$$
Hence, $\Omega_{d,\xi}\geq \frac{1-\xi}{d+1} - \frac{d-1}{2}\cdot\frac{1-\xi^2}{d^2 - 1}
    =\frac{(1-\xi)^2}{2(d+1)}.$

\subsection{Proof of Claim \ref{lm: kappa deep cuts}}\label{app: proof of claim 2}
Recall that
    since $K_t^\mu = \Gamma_t + \mu B_2^d$, every $z\in K_t^\mu$ can be written as $z = \gamma + u$ where $\gamma\in\Gamma_t$ and $||u||_2\leq \mu$. Therefore, $$ X_{t,a_t}^\top z = X_{t,a_t}^\top (\gamma + u)\leq b'_t + \mu ||X_{t,a_t}||_2\leq b'_t + \mu x_{\max},$$
and similarly, $$-X_{t,a_t}^\top z \leq b_t + \mu x_{\max}.$$
Hence, $K_t^\mu \subseteq K_{t-1}^\mu \cap H_{t,\mu}^+\cap H_{t,\mu}^-$,
where 
\begin{align*}
    &H_{t,\mu}^+ := \{\gamma\,:\,X_{t,a_t}^\top\gamma\leq b'_t + \mu x_{\max}\},\\
    &H_{t,\mu}^- := \{\gamma\ :\ -X_{t,a_t}^\top\gamma\leq b_t + \mu x_{\max}\}
\end{align*}
Fix any $t\geq 1$. If $I_t(\xi) = 1$, then $\rho_t >0$. Since $0\in E_{t-1}$, by Cauchy--Schwarz $|\alpha_t|\leq \rho_t$. Furthermore, $b_t,b_t'\geq 0$. It follows that
\begin{align*}
    &q_t^+ = \frac{b_t' + \mu x_{\max}-\alpha_t}{\rho_t}\geq \frac{\mu x_{\max}}{\rho_t}-1>-1,\\
    &q_t^- = \frac{b_t + \mu x_{\max}+\alpha_t}{\rho_t}\geq \frac{\mu x_{\max}}{\rho_t}-1>-1.
\end{align*}
If $q_t^+\leq\frac{\xi}{d}$, then Corollary \ref{cor: deep cut volume} applies to $E_{t-1}\cap H_{t,\mu}^+$. Since $K_t^\mu\subseteq E_{t-1}\cap H_{t,\mu}^+$, the monotonicity of MVEE volumes yields
\begin{align*}
    \Phi_{t-1}\!-\!\Phi_t\!\geq\!\log\vol(\!E_{t-1}\!)\!-\!\log\vol(\!M(E_{t-1}\!\cap\! H_{t,\mu}^+\!))\!\geq\! \Omega_{d,\xi}.
\end{align*}
The case of $\left\{q_t^- \leq \frac{\xi}{d}\right\}$ follows by applying the same argument to $H_{t,\mu}^-$. If $I_t(\xi) = 0$, the inclusion $K_t^\mu\subseteq K_{t-1}^\mu$ yields $\vol(E_t)\leq \vol(E_{t-1})$. Hence, $\Phi_{t-1}-\Phi_t \geq 0$. This also includes the $W_t=0$ scenario. The above holds for $t=1$ since $E_0 = M(K_0^\mu)$.

\subsection{Proof of Lemma \ref{lm: freedman}}\label{app: proof of lm 2}
First, we define a filtration which, at each round $t$, encodes all the historical information of this contextual bandit process. Formally, define
\begin{equation}\label{eq: filtration}
    \mathcal{F}_{t-1}\!:=\!\sigma\left\{\eta_s\!:\!1\leq s<t;\!X_{\ell,a}\!:\! 1\!\leq\!\ell\!\leq t,\!1\leq\!a\leq\!K\right\}.
\end{equation}

    Notice that the sets $K_{t-1}^\mu$, their MVEEs $E_{t-1}$, the chosen contexts $X_{t,a_t}$, and the weighted widths $W_t$ are all $\mathcal{F}_{t-1}$ measurable. Moreover, the $\xi$-deep-cut indicators $I_t(\xi)$ are $\mathcal{F}_t$ measurable. Recall \eqref{eq: partition}. It is trivial that $\left|\{1\leq t\leq T\,:\,W_t< w_\mu\}\right|\leq T$. Therefore, $\sum_{t\,:\,W_t<w_\mu}W_t\leq w_{\mu}T$. It suffices to further upper bound $\sum_{t\,:\,W_t\geq w_\mu}W_t$. For convenience, denote $j_t^\mu(\xi) := \mathbbm{1}\{W_t\geq w_\mu\}I_t(\xi),\ p_t^\mu(\xi) := \mathbb{E}\left(j_t^\mu(\xi)\middle|\mathcal{F}_{t-1}\right)$.
    
    Notice that $K_{t-1}^\mu\subseteq K_0^\mu \subseteq r_\mu B_2^d$, where $r_\mu := 2\theta_{\max}\sqrt{d}+\mu$. To bound its MVEE width, we utilize the contact-point characterization in \cite[Theorem~2.2]{GulerGurtuna2007}. There exists $z_1,\cdots,z_k\in K_{t-1}^\mu$ and $\lambda_1,\cdots,\lambda_k > 0$ such that
    \begin{align*}
        &\sum_{i=1}^k\lambda_i = d,\\
        &\sum_{i=1}^k\lambda_i(z_i-c_{t-1}) = 0,\\
        &B_{t-1} = \sum_{i=1}^k \lambda_i (z_i-c_{t-1})(z_i-c_{t-1})^\top.
    \end{align*}
Therefore,
\begin{equation*}
    \textup{tr}(B_{t-1}) = \sum_{i=1}^k \lambda_i ||z_i||_2^2 - d||c_{t-1}||_2^2\leq dr_\mu^2,
\end{equation*}
which implies that
\begin{equation*}
    W_t \leq 2||X_{t,a_t}||_2\sqrt{\textup{tr}(B_{t-1})}\leq 2x_{\max}\sqrt{d}r_\mu = W_{\max}^\mu.
\end{equation*}

It follows that $\forall W \in [0,W^\mu_{\max}]$, $W \leq \left(1 + \frac{W^\mu_{\max}}{w_\xi}\right)\min\{W, w_\xi\}$.
Taking sum yields
\begin{equation}\label{eq: trunc}
    \sum_{t:\ W_t\geq w_\mu}W_t\leq \left(1 + \frac{W^\mu_{\max}}{w_\xi}\right)\sum_{t:\ W_t\geq w_\mu}\min\{W_t,w_\xi\}.
\end{equation}
    We make the following claim.
    
    \begin{claim}
    [Upper bound on large terms]\label{lm: lb pt}
    For every $t\geq 1$, $p_t^\mu(\xi) \geq \mathbbm{1}\{W_t\geq w_\mu\}c_0\min\left\{\epsilon_0,\frac{\xi W_t}{4d}\right\}$.
\end{claim}

The proof of Claim \ref{lm: lb pt} is in Appendix \ref{app: proof of claim 3}. Taking sum over $t$ and factorizing $\frac{\xi}{4d}$ inside the minimum, we have
\begin{equation}\label{eq: bound on sum pt}
    \sum_{t:\ W_t\geq w_\mu}\min\left\{w_\xi, W_t\right\}\leq \frac{4d}{c_0\xi}\sum_{t=1}^T p_t^\mu(\xi).
\end{equation}

Next, we state and prove a probabilistic proposition providing high-probability bound on terminal martingales, whose proof is in Appendix \ref{app: proof of prop 1}.
\begin{proposition}[terminal martingale]\label{prop: terminal}
    Let $j_t\in\{0,1\}$ be $\mathcal{F}_t$-measurable, $p_t := \mathbb{E}\left(j_t\middle| \mathcal{F}_{t-1}\right)$. Define
    $J_T :=\sum_{t=1}^T j_t$, $P_T := \sum_{t=1}^T p_t$.
    For every $\delta\in(0,1)$, $\mathbb{P}\left(P_T\leq 2J_T + \frac{8}{3}\log\frac{1}{\delta}\right) \geq 1-\delta$.
\end{proposition}

Applying Proposition \ref{prop: terminal} to $j_t^\mu(\xi)$, we obtain that with probability no less than $1-\delta$,
\begin{equation}\label{eq: p bound}
        \sum_{t=1}^T p_t^\mu(\xi)\!\leq\!2\sum_{t=1}^T j_t^\mu(\xi)\!+\!\frac{8}{3}\log\frac{1}{\delta}\!\leq\! 2\sum_{t=1}^T I_t(\xi)\!+\!\frac{8}{3}\log\frac{1}{\delta}.
\end{equation}
Combining equations \eqref{eq: trunc}, \eqref{eq: bound on sum pt}, and \eqref{eq: p bound}, we complete the proof of Lemma \ref{lm: freedman}.

\subsection{Proof of Claim \ref{lm: lb pt}}\label{app: proof of claim 3}
    First notice that on the event $\{W_t<w_\mu\}$, $p_t^\mu(\xi)=0$ holds. Otherwise, on the event $\{W_t\geq w_\mu\}$, we have $\rho_t>0$. By algebraic manipulations, the event $\left\{q_t^+ \leq \frac{\xi}{d}\right\}$ is equivalent to $\{\eta_t\leq -\eta_{\max}+u_t\}$,
where $u_t = \alpha_t + \frac{\xi \rho_t}{d} - \mu x_{max}$; the event $\left\{q_t^- \leq \frac{\xi}{d}\right\}$ is equivalent to $\{\eta_t\geq \eta_{\max} - v_t\}$,
where $v_t = \frac{\xi \rho_t}{d} - \alpha_t - \mu x_{max}$. Therefore, $u_t + v_t = \frac{\xi W_t}{d} - 2\mu x_{max}$.
If $W_t \geq w_\mu = \frac{4d\mu x_{max}}{\xi}$, then $u_t + v_t \geq \frac{\xi W_t}{2d}$.
Thus, $\max\{u_t,v_t\} \geq \frac{\xi W_t}{4d}$. Recall that $u_t$, $v_t$, and $W_t$ are $\mathcal{F}_{t-1}$-measurable. Denote $\epsilon_t:= \min\left\{\epsilon_0,\frac{\xi W_t}{4d}\right\}$ which is $\mathcal{F}_{t-1}$-measurable. On the event $\{W_t\geq w_\mu\}$, we have $0<\epsilon_t\leq \epsilon_0$. By Assumption \ref{ass: boundary-visiting},
\begin{align*}
    &\mathbb{P}\left(\eta_t \leq -\eta_{\max}+\epsilon_t\middle| \mathcal{F}_{t-1}\right)\geq c_0\epsilon_t,\\
    &\mathbb{P}\left(\eta_t \geq \eta_{\max}-\epsilon_t\middle| \mathcal{F}_{t-1}\right)\geq c_0\epsilon_t.
\end{align*}
On the $\mathcal{F}_{t-1}$-measurable event $\left\{W_t\geq w_\mu,\ u_t\geq \frac{\xi W_t}{4d}\right\}$, the event $\{\eta_t\leq -\eta_{\max}+\epsilon_t\}$ implies $\left\{q_t^+ \leq \frac{\xi}{d}\right\}$. Otherwise, on the event $\left\{W_t\geq w_\mu,\ u_t<\frac{\xi W_t}{4d}\right\}$, we have $v_t\geq \frac{\xi W_t}{4d}\geq \epsilon_t$, so the event $\{\eta_t \geq \eta_{\max}-\epsilon_t\}$ implies $\left\{q_t^- \leq \frac{\xi}{d}\right\}$. Therefore,
\begin{equation*}
    p_t^\mu(\xi) \geq \mathbbm{1}\{W_t\geq w_\mu\}c_0\min\left\{\epsilon_0,\frac{\xi W_t}{4d}\right\}.
\end{equation*}

\subsection{Proof of Proposition \ref{prop: terminal}} \label{app: proof of prop 1}
 Set $X_t := p_t - j_t,\quad M_t := \sum_{s=1}^t X_s$.
Then 
\begin{align*}
    & M_T = P_T - J_T,\\
    &\mathbb{E}\left(X_t\middle| \mathcal{F}_{t-1}\right) = 0,\\
    &|X_t|\leq 1.
\end{align*}
Define $V_t := \sum_{s=1}^t\mathbb{E}\left(X_s^2\middle| \mathcal{F}_{s-1}\right)$. Since $j_t$ is conditionally Bernoulli with mean $p_t$, then
$$\mathbb{E}\left(X_t^2\middle| \mathcal{F}_{t-1}\right) = p_t(1-p_t) \leq p_t.$$
Hence, $V_T\leq P_T$. Notice that $\forall\, u < 3,\ \exp(u)\leq 1 + u + \frac{u^2}{2 - 2u/3}$.
Fix $\lambda \in (0,3)$, we have
$$\exp(\lambda X_t) \leq 1 + \lambda X_t + \psi(\lambda)X_t^2,$$
where $\psi(\lambda) := \frac{\lambda^2}{2-2\lambda/3}$. The sequence $\{X_t\}_t$ is a martingale difference sequence with respect to $\{\mathcal{F}_t\}_t$, and $\{M_t\}_t$ is the corresponding martingale. Taking conditional expectations:
\begin{align*}
    \mathbb{E}\left(\exp(\lambda X_t)\mid \mathcal{F}_{t-1}\right) &\leq 1 + \psi(\lambda)\mathbb{E}\left(X_t^2\middle|\mathcal{F}_{t-1}\right)\\
    &\leq \exp\left[\psi(\lambda)\mathbb{E}\left(X_t^2\middle|\mathcal{F}_{t-1}\right)\right].
\end{align*}
Therefore, $Z_t := \exp\left[\lambda M_t - \psi(\lambda)V_t\right]$
is a non-negative supermartingale. It follows that $$\mathbb{E}(Z_T) \leq \mathbb{E}(Z_0) = 1.$$
Hence, $\mathbb{E}\exp(\lambda M_T - \psi(\lambda) P_T) \leq 1$, which implies $\mathbb{E}\exp((\lambda-\psi(\lambda))P_T - \lambda J_T)\leq 1$.
Using Markov's inequality, with probability no less than $1-\delta$,
$$(\lambda - \psi(\lambda))P_T - \lambda J_T \leq \log\frac{1}{\delta}.$$
Pick $\lambda = \frac{3}{4}$. Then
$P_T \leq 2J_T + \frac{8}{3}\log\frac{1}{\delta}$.

\bibliographystyle{IEEEtran}
\bibliography{bibliography}

\end{document}